\newcommand{\lam}{\ensuremath{{\lambda}}}
\newcommand{\hbet}{\ensuremath{\hat \beta(\lam)}}

\newcommand{\X}{\ensuremath{{\mathbf{X}}}}

\newcommand{\Y}{\ensuremath{{Y}}}
\newcommand{\e}{\ensuremath{{\epsilon}}}

\newcommand{\minbeta}{\ensuremath{{M(\truebeta)}}}

\newcommand{\Xa}{\ensuremath{{X_S}}}
\newcommand{\XaT}{\ensuremath{{X_S^T}}}

\newcommand{\XbT}{\ensuremath{{X_{S^c}^T}}}

\newcommand{\truebeta}{\ensuremath{{\beta^*}}}
\newcommand{\truebetaj}{\ensuremath{{\beta_j^*}}}

\newcommand{\betA}{ {\beta_S^*}}

\newcommand{\hbetA}{\ensuremath{\hat{\beta}^{(1)}}}

\documentclass[aos,preprint,noinf]{imsart}
\usepackage{geometry}                
\RequirePackage[OT1]{fontenc}
\RequirePackage{amsthm,amsmath}
\RequirePackage{natbib}
\RequirePackage[colorlinks,citecolor=blue,urlcolor=blue]{hyperref}
\usepackage{graphicx}
\usepackage{amsthm}
\usepackage{amssymb}
\usepackage{subfigure}
\usepackage{ulem}
\usepackage{color}
\usepackage{enumerate}

\newcommand{\R}{\ensuremath{{\mathbb{R}}}}

\startlocaldefs
\newtheorem{theorem}{Theorem}
\newtheorem{lemma}{Lemma}
\newtheorem{definition}{Definition}
\newtheorem{property}{Property}
\endlocaldefs

\begin{document}

\begin{frontmatter}
\title{On  Pattern Recovery of the Fused Lasso \thanksref{T0}}
\runtitle{Fused Lasso}

\thankstext{T0}{Supported partially by  the National Basic Research Program of China (973 Program 2011CB809105), NSFC-61121002, NSFC-11101005, DPHEC-20110001120113, and MSRA. }

\begin{aug}
\author{\fnms{Junyang} \snm{Qian }\thanksref{t1,t2}\ead[label=e1]{junyangq@sas.upenn.edu}}
\and
\author{\fnms{Jinzhu} \snm{Jia}\thanksref{t1}\ead[label=e2]{jzjia@math.pku.edu.cn}}



\runauthor{J. Qian et al.}

\affiliation{
Peking University \thanksmark{t1} and University of Pennsylvania  \thanksmark{t2}}

\address{Department of Mathematics \\
University of Pennsylvania \\
Philadelphia, PA 19104. \\
\printead{e1} \\
}

\address{School of Mathematical Sciences \\
and Center for Statistical Science, LMAM, LMEQF\\
         Peking University\\
         Beijing 100871, China \\
\printead{e2}\\
}

\end{aug}

\begin{abstract}
We study the property of the {Fused Lasso Signal Approximator (FLSA)} for estimating a blocky signal sequence with additive noise. We transform the {FLSA} to an ordinary Lasso problem. By studying the property of the design matrix in the transformed Lasso problem, {we find that the irrepresentable condition might not hold, in which case we show that the {FLSA} might not be able to recover the signal pattern.} We then apply the newly developed preconditioning method -- Puffer {T}ransformation \citep{jia2012preconditioning} on the transformed Lasso problem. We call the new method {the} {\textit {preconditioned fused Lasso}} and  we {give} non-asymptotic results for {this method}. Results show that when the signal jump strength (signal {difference} between two {neighboring} groups) is big and the noise level is small, our {preconditioned fused Lasso estimator} gives the correct pattern with high probability. Theoretical results give {insight} on what controls the signal pattern recovery ability -- it is the noise level {instead of} the length of the sequence.  Simulations confirm our theorems and show significant {improvement} of the preconditioned {fused} Lasso estimator over the vanilla {FLSA}.
\end{abstract}



\end{frontmatter}

\section{Introduction}
{Assume we have a sequence of signals $(y_1,y_2,\ldots, y_n)$ and it follows the linear model
\begin{equation}
\label{eqn:model}
y_i = \mu_i^* + \epsilon_i, \quad i = 1,2,\ldots, n,
\end{equation}
where $Y = (y_1,\ldots,y_n)^T \in \R^n$ is the observed signal vector, $\mu^* = (\mu_1,\ldots,\mu_n)^T\in \R^n$ the expected signal, and $\epsilon = (\epsilon_1,\ldots,\epsilon_n)^T$ is the white noise that is assumed to be i.i.d.\ and each has a normal distribution with mean $0$ and variance $\sigma^2$. The model is assumed to be sparse in the sense that the signals come in blocks and only a few of the blocks are nonzero. To be exact, there exists a partition of $\{1,2,\ldots,n\} = \cup_{j=1}^J [L_j,U_j]$ with $L_1 = 1, U_J = n, U_j\geq L_j, L_{j+1} = U_j +1$, and the following stepwise function holds:
\[\mu_i^* = \nu_j^* \mathbf{1}_{L_j\le \mu_i\leq U_j},\] with $\nu_j^*, L_j, U_j$ fixed but unknown. We also assume that the vector $\nu = (\nu_1,\nu_2,\ldots,\nu_n)$ is sparse, meaning that only a few of $\nu_j$'s are nonzeros. We point out that the Gaussian noise is not necessary. But we still use it to study the {insight} of the fused Lasso. The variance {$\sigma^2$ of $\epsilon_i$ is the measure of noise level and} does not have to be a constant here. For a lot of real data problem, each observation of $y_i$ can be an average of multiple measurements and so $\sigma^2$ decreases when the number of measurements increases.

This model {featured by blockiness and sparseness} has many applications. For example, in tumor studies, based on the Comparative Genomic Hybridization(CGH) data, it can be used to {automatically} detect the gains and losses in DNA copies by taking ``signals" as the {log-ratio between} the number of DNA copies in the tumor cells and that in the reference cells \citep{tibshirani2008spatial}. }

One way to estimate the unknown parameters is via {the Fused Lasso Signal Approximator (FLSA)} defined as follows \citep{tibshirani2004sparsity, friedman2007pathwise}:

\begin{equation}
\label{eqn:fusedLasso}
\hat\mu (\lambda_1,\lambda_2)= {\mathop{\text{argmin}} \limits_{\mu}}  \frac{1}{2}\|Y-\mu\|^2_2 + \lambda_1\|\mu\|_1 + \lambda_2\|\mu\|_{TV},
\end{equation}
where $\|\mu\|_1 = \sum_{i=1}^n |\mu_i|$,$\|\mu\|_2^2 = \sum_{i=1}^n \mu_i^2$ and $\|\mu\|_{TV} = \sum_{i=1}^{n-1} |\mu_{i+1} - \mu_i|$. The {$L_1$-}norm regularization controls the sparsity of the signal and the total variation {seminorm} ($\|\mu\|_{TV}$) regularization controls the number of blocks (partitions or groups).

{Figure \ref{fig:gene} gives one example of {the} signal sequence and the {FLSA estimate} on CGH data.  More details and examples can be seen in \cite{tibshirani2008spatial}.}

\begin{figure}[h!]
\begin{center}
  \includegraphics[scale = 0.6]{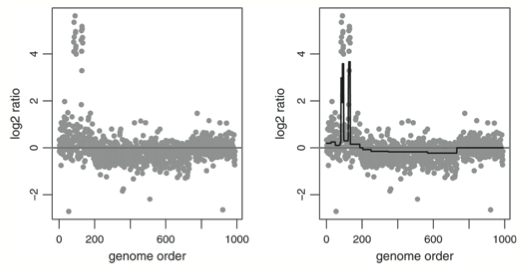}
  \caption{\label{fig:gene}This figure is from \cite{tibshirani2008spatial}. {The fused} Lasso is applied to some CGH data. The data are shown in the left panel, and the solid line in the right panel represents {the estimated signals by the fused} Lasso. The horizontal  line is for $y = 0$.
}
\end{center}
\end{figure}

One important question for the FLSA is how good the estimator defined in Equation \eqref{eqn:fusedLasso} is. {We analyze in this paper} if the FLSA can recover the ``stepwise pattern" or not. We also {try to} answer the following question: what do we do if the FLSA does not recover the ``stepwise pattern"? To measure how good an estimator is{, we introduce the following definition of Pattern Recovery.}

\begin{definition}[Pattern Recovery]
  An FLSA {solution} $\hat\mu(\lambda_{1n},\lambda_{2n})$  recovers the signal pattern if and only if there exists $\lambda_{1n}$ and $\lambda_{2n}$, such that
\begin{equation}
\label{eqn:js}
 sign(\hat \mu_{i+1}(\lambda_{1n},\lambda_{2n}) - \hat \mu_{i}(\lambda_{1n},\lambda_{2n}))  = sign(\mu_{i+1} - \mu_i), \quad {i = 1, \ldots, n-1}.
\end{equation}
\end{definition}

We use $\hat\mu =_{js} \mu$ to shortly denote Equation \eqref{eqn:js}{ ($js$ is the acronym for $jump\ sign$)}. {The FLSA} with the property of pattern recovery means that it can be used to identify both the groups and jump directions (up or down) between groups.

The concept of pattern recovery is very similar to sign recovery of  the Lasso. In fact, some simple calculations in Section \ref{sec:FSLA-LASSO} tell us that {the} pattern recovery property of {the FLSA} can be transformed to {the} sign recovery property of the Lasso estimator.

For observation pairs $(x_i,y_i), i =1,2,\ldots,n$ with ${x_i}\in \R^p$ and $y_i\in \R$, the Lasso is defined as follows \citep{tibshirani1996regression}.
\begin{equation*}
\hat\beta(\lambda) = {\mathop{\text{argmin}}\limits_{\beta}} \frac{1}{2} \sum_{i=1}^n (y_i - x_i^T\beta)^2 + \lambda \|\beta\|_1.
\end{equation*}
Equivalently, in matrix form,
\begin{equation}
\label{eqn:Lasso}
\hat\beta(\lambda) = {\mathop{\text{argmin}}\limits_{\beta}} \frac{1}{2} \|Y - X\beta\|^2 + \lambda \|\beta\|_1,
\end{equation}
where $Y = (y_1,\ldots,y_n)^T$ and $X \in \R^{n\times p}$ with ${x_i^T}$ as its $i$th row. We use $X_j$ to denote the $j$th column of $X$.

Sign {Recovery} of the Lasso estimator is defined as follows.
\begin{definition}[Sign Recovery]
Suppose that data $(X,Y)$ {follow} a linear model: $Y = X\beta^* + \epsilon$, where $Y = (y_1,\ldots,y_n)^T$, $X \in \R^{n\times p}$ with ${x_i^T}$ as its $i$th row, $\beta^*\in \R^{p\times 1}$ and $\epsilon = (\epsilon_1,\ldots,\epsilon_n)^T \in \R^{n\times 1}$ with $E(\epsilon_i) = 0$. A Lasso estimator $\hat\beta(\lambda_{n})$ has the sign recovery property if and only if there exists $\lambda_{n}$  such that
\begin{equation}
\label{eqn:seq}
 {sign(\hat \beta_j(\lambda_{n}))}   = sign(\beta_j^*), \quad {j = 1, \ldots, p}.
\end{equation}
\end{definition}

We will use $\hat\beta =_{s} \beta^*$ to shortly denote $sign(\hat \beta_j(\lambda_{n}){ )} = sign(\beta_j^*),\ j = 1, \ldots, p.$ The Lasso estimator with {the} sign recovery property  implies that it selects the correct set of predictors.  If $P(\hat\beta(\lambda_{n}) =_{s} \beta)\rightarrow 1$, as {the} sample size $n\rightarrow \infty$, we say that $\hat\beta(\lambda_{n})$ is sign consistent.

A rich theoretical literature has studied the consistency of the Lasso, highlighting several potential pitfalls   \citep{knight2000asymptotics, fan2001variable, greenshtein2004persistence, donoho2006stable, meinshausen2006high, tropp2006just, zhao2006model, zhang2008sparsity, wainwright2009}.   The sign consistency of the Lasso  requires the irrepresentable condition, a stringent assumption on the design matrix \citep{zhao2006model}.  Now it is well understood that if the design matrix  violates {the} irrepresentable condition, {the Lasso} will perform poorly and the estimation performance will not {be improved} by increasing the sample size.

Our study {of} the pattern recovery of the {FLSA} begins with a transformation that changes the FLSA to a special {Lasso problem. The data defined in the transformed Lasso problem has correlated {noise} terms instead of independent { ones}. We prove that even for {the} linear model with correlated noise, the irrepresentable condition is still necessary for sign consistency.}  We then analyze the property of the design matrix in the transformed Lasso problem.  We give necessary and sufficient condition such that the design matrix in the transformed Lasso problem satisfies the irrepresentable condition.  We show that,  for a special class of models (with special designed stepwise function on $\mu_i^*$),  the irrepresentable condition holds.  For other {signal patterns}, the irrepresentable condition does not hold and thus the FLSA may fail to keep consistent.  A recent paper ``Preconditioning to comply with the irrepresentable condition" by \cite{jia2012preconditioning} {shows} that a Puffer Transformation will improve the Lasso and make the Lasso estimator sign consistent under some mild conditions. We apply this technique, {propose the \textit{preconditioned fused Lasso}} and show that it improves the {FLSA} {and} recovers the signal pattern with high probability.

{ In \cite{rinaldo2009properties}, the author also considers the consistency conditions for the FLSA. They showed that under some  conditions, the FLSA can be consistent both in block reconstruction and model selection. 
The author says in \cite{rinaldo2009properties} that the asymptotic results may have little guidance to the practical performance when $n$ is finite. However, our method, {as we will see}, can not only provide mild {conditions for the estimator to be consistent in block recovery but also give} an explicit non-asymptotic lower bound on the probability that the true blocks are recovered. Numerical simulations also illustrate that in many cases our method turns out to be more effective in block recovery.}

The rest of the paper is organized as follows. In Section 2, we  transform the {FLSA} problem into a Lasso problem and analyze the property of the design matrix in the transformed Lasso problem.  Section 3 illustrates when the {FLSA} can recover the signal pattern and when it cannot.  In Section 4, we propose a new algorithm called the preconditioned fused Lasso that improves {the} FLSA by the technique of Puffer {Transformation} (defined in Equation \eqref{eqn:PT}). We show that for a wide range designs of the stepwise function on $\mu^*$, {this algorithm} can recover the signal pattern with high probability.  In Section 5, simulations are implemented to compare the performances between the preconditioned fused Lasso and the vanilla {FLSA}. Section 6 concludes the paper. Some proofs are given in the appendix.

\section{{FLSA} and the Lasso}
\label{sec:FSLA-LASSO}
{ We turn the FLSA problem into a Lasso problem by change of variables.}
Define {the} soft thresholding function $SH_{\lambda}(x)$ as
\[
SH_{\lambda}(x)=\left\{
   \begin{array}{cc}
      x + \lambda & x < -\lambda \\
      0 & -\lambda\leq x \leq \lambda \\
      x - \lambda & x > \lambda \\
    \end{array}
         \right.
.\]

Let $\hat\mu(\lambda_1,\lambda_2)$ be the fused Lasso estimator defined in Equation \eqref{eqn:fusedLasso}. We have the following result.

\begin{lemma}
\label{lemma:lam1=0}
\[\hat\mu(\lambda_1,\lambda_2) = SH_{\lambda_1}(\hat\mu(0,\lambda_2) )\]
\end{lemma}

The proof of Lemma \ref{lemma:lam1=0} can be found in \cite{friedman2007pathwise}.  From Lemma \ref{lemma:lam1=0}, to study the property of $\hat\mu(\lambda_1,\lambda_2)$, we can set $\lambda_1 = 0$ first. In the whole paper, {since pattern recovery is our main concern,} so we only consider the case when $\lambda_1 = 0$. When $\lambda_1 = 0$, we can  solve the {FLSA} by {change of variables}. Let $\theta_1 = \mu_1, \theta_i = \mu_i - \mu_{i-1}, i =2,\ldots,n.$ In matrix form, we have
$\mu = A\theta$, with
\begin{equation}
\label{eqn:A}
A_{n\times n} =    \left(\begin{matrix} 
      1 & 0 & \ldots &\ldots &\ldots &0 \\
      1 & 1 & 0  &\ldots & \ldots &0 \\
      1 & 1 & 1 &0 &\ldots &0\\
      \ldots & \ldots&\ldots &\ldots &\ldots &\ldots &\\
      1& 1& 1& \ldots&\ldots &1
   \end{matrix}
   \right)
\end{equation}

So by using $\theta$ instead of $\mu$, we have an equivalent solution of $\hat \mu(0,\lambda_2)$ via the following $\hat \theta(\lambda_2)$.
\begin{equation}
\label{eqn:FLSA2Lasso}
\hat\theta(\lambda_2) = \mathop{\text{argmin}} \limits_{\theta} \frac{1}{2}\|Y-A\theta\|_2^2 + \lambda_2 \|\tilde\theta\|_1,
\end{equation}
where ${\tilde \theta= (\theta_2,\theta_3,\ldots,\theta_n)^T\in \R^{n-1}}$. Once we obtain $\hat\theta(\lambda_2)$, we have $\hat\mu(0,\lambda_2) = A\hat\theta(\lambda_2)$. Notice the special form of the design matrix $A$, Expression \eqref{eqn:FLSA2Lasso} is a  Lasso problem with interception. In fact, Expression \eqref{eqn:FLSA2Lasso}  can be rewritten as

\begin{equation} \label{trans flsa}
  {\hat{\theta}}(\lambda_2) = \mathop{\text{argmin}} \limits_{(\theta_1, {\tilde{\theta}})} \frac{1}{2}{\|Y - \theta_1 - {X}{\tilde{\theta}}\|_2}^2 + \lambda_2 \|{\tilde{\theta}}\|_1.
\end{equation}
where ${\tilde{\theta}} = (\theta_2, \ldots, \theta_n)^T$ and ${X} = (x_{ij}) \in \mathbb{R}^{n \times (n-1)}$:
\begin{equation}
\label{eqn:X}
  x_{ij} = \begin{cases}
    1 & i > j \\
    0 & i \leq j.
  \end{cases}
\end{equation}

Define the centered version of $X\in \R^{n\times (n-1)}$ and $Y\in \R^{n}$ as follows.
\begin{equation}
\label{eqn:centered}
\tilde X = [X_1-\bar{X_1}, \ldots, X_{n-1} - \bar X_{n-1}] \mbox{  and  }\tilde Y = Y - \bar Y
\end{equation}
with $\bar u$ being the average of the vector $u$. It is easy to see that Expression \eqref{trans flsa} is equivalent to the following standard Lasso problem without interception.

\begin{equation} \label{trans flsa1}
  {\hat{\tilde \theta}}(\lambda_2) = \mathop{\text{argmin}} \limits_{ {\tilde{\theta}}}  \frac{1}{2}{\|\tilde Y  - \tilde{X}{\tilde{\theta}}\|_2}^2 + \lambda_2 \|{\tilde{\theta}}\|_1, \ \ \mbox{and} \ \ \hat\theta_1(\lambda_2) = \bar Y - \bar X \hat{\tilde \theta}(\lambda_2).
\end{equation}

Since the observation $Y = (y_1,\ldots,y_n)$ follows {the} model defined {in} Equation \eqref{eqn:model}. Define $\theta^* = A^{-1} \mu^*$, (equivalently, $\theta_1^* = \mu_1^*, \theta_i ^*= \mu_i ^*- \mu_{i-1}^*, i =2,\ldots,n$), where $A$ is defined in Equation \eqref{eqn:A}.  Let $\tilde \theta^* \in \R^{n-1}= (\theta_2^*,\theta_3^*,\ldots,\theta_n^*)^T$. We have that $(X,Y)$
satisfy the following linear model:
\[Y = A\theta^* + \epsilon = \theta_1^* + X\tilde \theta^* + \epsilon,\]
where $X$ is defined at Equation \eqref{eqn:X}.
 Consequently  the centered version of $(X,Y)$ satisfy the following linear model:
\begin{equation}
\label{eqn:LM}
\tilde Y =  \tilde X \tilde \theta^* + \tilde \epsilon,
\end{equation}
where $\tilde \epsilon = \epsilon - \bar{\epsilon}$ with  $E(\tilde \epsilon) = 0$.
Now we see that $\hat{\tilde \theta}(\lambda_2) $ defined at \eqref{trans flsa1} has the sign recovery property if and only if $\hat{\tilde \theta}(\lambda_2) =_s \tilde \theta^*.$ By the relationship between $\mu$ and $\theta$,  $\hat{\tilde \theta}(\lambda_2) =_s \tilde \theta^*$ is equivalent to $\hat\mu(0,\lambda_2) =_{js} \mu^*$. {In} other words, the pattern recovery property of an {FLSA}  can be viewed as sign recovery of a Lasso estimator.
\begin{property}
The pattern recovery of the {FLSA} $\hat\mu (0,\lambda_2)$ defined in Equation \eqref{eqn:fusedLasso} is equivalent to the sign consistency of the the Lasso estimator $\hat{\tilde \theta}(\lambda_2)$ defined in Equation \eqref{trans flsa1}.
\end{property}

{
Note that this change of variables serves mainly for {theoretical} analysis rather than computational facilitation.
Although there are many mature algorithms for the Lasso, {transforming} the {FLSA} to the Lasso is not recommended in {practice} because it makes the design matrix in \eqref{trans flsa1} much more dense{, which is unfavorable to the efficiency of computation}. Instead, \cite{friedman2007pathwise} develops specialized algorithm for the FLSA based on the coordinate-wise descent. \cite{hoefling2010path} generalizes the path algorithm and extends it to the general fused Lasso problem.
However, in our consistency analysis, this transformation works since we can use the well understood techniques on the Lasso to analyze the theoretical properties of the FLSA.
}

We now turn to analyze the Lasso problem defined in Equation \eqref{trans flsa1}.

\section{{ {The} Transformed  Lasso}}
It is now well understood that in a standard linear regression problem the Lasso is sign consistent when the design matrix satisfies some stringent conditions. One such condition is {the} irrepresentable condition defined as follows.

\begin{definition}[Irrepresentable {Condition}]
The design matrix $X$ satisfies the \textbf{Irrepresentable {Condition}} for $\truebeta$ with support $S = \{j: \truebetaj \neq 0\}$ if, for some  $\eta \in (0,1]$,
\begin{equation}
\left\|\XbT\Xa\left(\XaT\Xa\right)^{-1}sign(\betA)\right\|_{\infty}\leq
1-\eta, \label{IC}
\end{equation}
where for a vector $x$, $\|x\|_{\infty}  = \max_i |x_i|$, and for $T \subset \{1, \dots, p\}$ with $|T| = t$,  $\X_T \in \R^{n \times t}$ is a matrix which containes the columns of $\X$ indexed by $T$.
\end{definition}

Let {$\Lambda_{\min}\left(X\right)$ be the minimal eigenvalue of the matrix $X$ and }
\begin{equation*}
C_{\min} = \Lambda_{\min}\left(\Xa^T\Xa\right) > 0.
\end{equation*}

Define
$$\Psi(\X, \truebeta, \lam) =\lam\left[ \frac{\eta}{\sqrt{C_{\min}}\max_{j\in S^c} \|X_j\|_2}+ \left\|\left(\Xa^T\Xa \right)^{-1}
sign(\beta^*_S)\right\|_\infty \right].$$
With the above notation, we have a general non-asymptotic result for the sign recovery of the Lasso when data $(X,Y)$ follow a linear model.
\begin{theorem}
\label{thm:consis}
Suppose that data $(\X,\Y)$ follow a linear model $Y = \X\truebeta + \epsilon$, where $Y = (y_1,\ldots,y_n)^T \in \R^{n\times1}$, $X \in \R^{n\times p}$ with ${x_i^T}$ as its $i$th row, $\beta^*\in \R^{p\times 1}$ and $\epsilon = (\epsilon_1,\ldots,\epsilon_n)^T \in \R^{n\times 1}$ with $\epsilon\sim N(0,\Sigma_\e)$.  Assume that the irrepresentable condition (\ref{IC})
holds.
If $\lam$ satisfies
$$\minbeta>  \Psi(\X, \truebeta, \lam),$$
then with probability greater than
$$1-2p
\exp\left\{-\frac{\lam^2\eta^2}{2[\Lambda_{\max}
(\Sigma_\e)\max_{j\in S^c} \|X_j\|_2^2]}\right\},$$
the Lasso has a unique solution $\hbet$ with $\hbet =_s \truebeta$.
\end{theorem}

The proof of Theorem \ref{thm:consis} is very similar to that {of} Lemma 3 in \cite{jia2012preconditioning} (pp. 24). The only difference is that in  \cite{jia2012preconditioning}, they scale each column of $X$ to be bounded with $\|X_j\|_2 \leq 1$. Here we do not have any assumption for the $\ell_2$ norm of $X_j$. If we further have the assumption that  $\|X_j\|_2 \leq 1$ for each $j$, then we have exactly the same result as in  \cite{jia2012preconditioning}. So we omit the proof for Theorem \ref{thm:consis} .

{The} irrepresentable condition is a key condition for the Lasso's sign consistency.  A lot of researchers noticed that the irrepresentable condition is a necessary condition for the Lasso's sign consistency \citep{zhao2006model,wainwright2009,jia2010Lasso}. We also state this conclusion  under a more general linear model with correlated noise terms.

\begin{theorem}
\label{THM:NECESSARY CONDITION}
Suppose that data $(\X,\Y)$ follow a linear model $Y = \X\truebeta + \epsilon$, with Gaussian noise $\epsilon\sim N(0,\Sigma_\e)$.  The irrepresentable condition (\ref{IC}) is necessary for the {sign consistency of the Lasso}. In other words, if
\begin{equation}
\left\|\XbT\Xa\left(\XaT\Xa\right)^{-1}sign(\betA)\right\|_{\infty}\geq 1,
\end{equation}
we have
\[P(\hbet =_s \truebeta) \leq \frac{1}{2}.\]
\end{theorem}

A proof of Theorem \ref{THM:NECESSARY CONDITION} can be seen in the appendix. Theorem \ref{THM:NECESSARY CONDITION} says that if the irrepresentable condition does not hold, {it} is very likely that the Lasso does not recover signs of the coefficients.

{With the above theorem, we now come back to the transformed Lasso problem defined in Equation \eqref{trans flsa1}} and examine if the irrepresentable condition holds or not {in this case}.
Recall that for the Lasso problem transformed from the {FLSA}, we have the design matrix
$$\tilde X = [X_1-\bar{X_1}, \ldots, X_{n-1} - \bar X_{n-1}].$$


Denote $S = \{j: \tilde \theta^*_j \neq 0\}$ as the index set of the relevant variables in the true model. Let $j$ be the index of any of the irrelevant variables. Then (\ref{IC}) can be written as
\begin{equation*}
\left|{\tilde{X_j}}^T {\tilde{X_S}} ({\tilde{X_S}}^T {\tilde{X_S}})^{-1}sign(\tilde\theta^*)\right| < 1, \forall j\not\in S
\end{equation*}
which is equivalent to
\[|\hat b_j^T sign(\tilde\theta^*)| < 1, \forall j\not\in S\]
with $\hat b_j\in \R^{|S|}$ the OLS estimate of $b_j$  in  the following linear regression equation
\begin{equation}
  \tilde{X_j} =  {b_j}^T \tilde{X}_S + \epsilon.
\end{equation}
 Since $\tilde X$ is the centered version of $X$, it can be easily shown that  $\hat b_j$ is also the OLS estimate of $b_j$ in the following linear regression equation:
\begin{equation} \label{reg eqn}
  X_j = b_0 +  {b_j}^T  {X_S} + \epsilon,
\end{equation}
where $b_0\in \R$ is the intercept term.

A stronger version of irrepresentable condition is as follows
\begin{equation}
\label{SIrC}
\left\|{\tilde{X_j}}^T {\tilde{X_S}} ({\tilde{X_S}}^T {\tilde{X_S}})^{-1}\right\|_1 < 1, \forall j\not\in S.
\end{equation}
If \eqref{SIrC} holds, then for any $\mu^*$ (equivalently, for any $\tilde \theta^*$) {the} irrepresentable condition always holds. Otherwiese, if \eqref{SIrC} does not hold, then there exists some $\tilde \theta^*$ such that {the} irrepresentable condition {fails to} hold. We have a necessary and sufficient condition on $\mu^*$ such that the stronger version of {the} irrepresentale condition \eqref{SIrC} holds.
\begin{theorem}
\label{thm:SIC}
  Assume $y=(y_1,\ldots,y_n)$ satisfies model \eqref{eqn:model}, {the collection of the indexes of jump points is} $S = \{j_1,j_2,\ldots,j_s\}$ with $j_k (1 \leq k \leq s)$  increasing.  Formally,
  $S = \{j: \mu_j^* \neq \mu_{j-1}^*, j =2,\ldots,n\}$.  Then the stronger version of {the} irrepresentable condition \eqref{SIrC} holds if and only if the jump points are consecutive. That is,
  $s = 1$ or
  \begin{equation*}
    \max_{1 \leq k < s} (j_{k+1} - j_k) = 1.
  \end{equation*}
\end{theorem}

\begin{proof}
Note that the OLS estimate of the coefficients in the linear regression equation (\ref{reg eqn}) is
\begin{equation}
  \left(\begin{array}{c}
    \hat b_0 \\
     {\hat b_j}
  \end{array}\right)
  = ( {Z_S}^T  {Z_S})^{-1}  {Z_S}^T  {X}_j,
\end{equation}
where $
   {Z_S} =
  \left(\begin{array}{cc}
     {1_n} & {X_S}
  \end{array}\right)$.
We know that $Z_S^T Z_S = (t_{k\ell}) \in \mathbb{R}^{(s+1)\times(s+1)}$ with
\begin{equation*}
  {t_{k\ell} = n - \max\{j_{k-1},j_{\ell-1}\}.}
\end{equation*}
%
where we assume $j_0 = 0$. According to a linear algebra result stated in Lemma \ref{lemmaapp} in  the appendix, the inverse of this matrix {is a tridiagonal matrix:
\begin{equation*}
    (Z_S^T Z_S)^{-1} = \left[\begin{array}{cccccc}
       r_{11} & r_{12} &  &  &  & \\
         r_{21} & r_{22} &  r_{23} & & &\\
         & r_{32} & r_{33}  &  r_{34}  & &\\
         & & \ddots & \ddots & \ddots & \\
         & & & r_{{s,s-1}} & r_{{s,s}} & r_{{s,s+1}} \\
         & & & & r_{{s+1,s}} & r_{{s+1,s+1}}
         \end{array}\right]
\end{equation*}
where
\begin{equation*}
      r_{k\ell} = \begin{cases}
      \frac{1}{j_1} & k = \ell = 1 \\
      -\frac{1}{j_{\ell-1} - j_{{\ell-2}}} & k = \ell - 1 \\
      -\frac{1}{j_\ell - j_{\ell-1}} & k = \ell + 1 \\
      \frac{j_{{\ell}} - j_{{\ell}-2}}{(j_{{\ell}-1} - j_{{\ell}-2})(j_{{\ell}} - j_{{\ell}-1})} & 1 < k = \ell < s+1 \\
      \frac{n-j_{s-1}}{(j_{s} - j_{s-1})(n-j_{s})} & k = \ell = s+1 \\
      0 & \text{otherwise}.
      \end{cases}
\end{equation*}

Denote $v = \left(\begin{array}{c}
    \hat b_0 \\
     {\hat b_j}
  \end{array}\right)
  =( {Z_S}^T  {Z_S})^{-1}  {Z_S}^T  {X}_j$. There are three pattern types that we need to consider.
\begin{enumerate}
\renewcommand{\labelenumi}{(\roman{enumi})}
\item
    If there exists $1 \leq k < s$ such that $j_{k+1} - j_k \geq 2$, then for any $j$  with $j_k<j<j_{k+1}$,
    \begin{equation*}
        {Z_S}^T {X}_j = (\underbrace{n-j,n-j,\ldots,n-j}_{{k+1}}, n-j_{k+1}, n-j_{k+2}, \ldots, n-j_s)^{T}.
    \end{equation*}

     We have
    \begin{equation*}
     v = \big(0, \ldots, 0, \frac{j_{k+1}-j}{j_{k+1}-j_k}, -\frac{j_{k+1}-j}{j_{k+1}-j_k} + 1, 0, \ldots,0 \big)^T.
    \end{equation*}
    Hence,
    \begin{equation}
    \label{eqn:samesign}
    ||\hat b_j||_1 = \displaystyle{\bigg |\frac{j_{k+1}-j}{j_{k+1}-j_k}\bigg | + \bigg |-\frac{j_{k+1}-j}{j_{k+1}-j_k} + 1\bigg | = 1}, \mbox{since $j_k < j < j_{k+1}$.}
    \end{equation}

\item
    If $j < j_1$,
    \begin{equation*}
        {Z_S}^T {X}_j = (n-j,n-j_1,n-j_2, \ldots, n-j_s)^{T}.
    \end{equation*}
  We have
    \begin{equation*}
        v = \big(1 - \frac{j}{j_1}, \frac{j}{j_1}, 0, \ldots, 0\big)^T.
    \end{equation*}
    Hence, $||\hat b_j||_1 = |\frac{j}{j_1}| < 1$.
\item
    If $j > j_s$,
    \begin{equation*}
        {Z_S}^T {X}_j = (\underbrace{n-j,\ldots, n-j}_{s+1})^{T}.
    \end{equation*}
    We have
    \begin{equation*}
    v = \big(0, \ldots, 0, \frac{n-j}{n-j_s}\big)^T.
    \end{equation*}
    Hence, $||\hat b_j||_1 = |\frac{n-j}{n-j_s} | < 1$.
    \end{enumerate}
}

These three cases for the position  of $j\in S^c$ show that as long as $j$ is not {between} two jump points, $\|\hat b_j\|_1 < 1$. Otherwise $\|\hat b_j\|_1 = 1$.  So   \begin{equation*}
   s = 1 \  \  \ \mbox{or} \ \ \ \max_{1 \leq k < s} (j_{k+1} - j_k) = 1
  \end{equation*} is necessary and sufficient for all $\|\hat b_j\|_1<1${, $j\in S^c$}.
\end{proof}

The above theorem shows that only a few special structures on $\mu^*$ make the stronger version of {the} irrepresentable condition {hold}. From the proof, we can propose a necessary and sufficient condition for  {the} irrepresentable condition.

\begin{theorem}
\label{thm:ic}
  Assume $y=(y_1,\ldots,y_n)$ satisfies model \eqref{eqn:model}, {the collection of the indexes of jump points} are $S = \{j_1,j_2,\ldots,j_s\}$ with $j_k (1 \leq k \leq s)$  increasing.  Formally,
  $S = \{j: \mu_j^* \neq \mu_{j-1}^*, j =2,\ldots,n\}$.  Then the  irrepresentable condition \eqref{IC} holds if and only if one of the following two conditions holds.
  \begin{enumerate}[(1)]
  \item  The jump points are consecutive. That is, $s=1$ or
  \begin{equation*}
    \max_{1 \leq k < s} (j_{k+1} - j_k) = 1.
  \end{equation*}

  \item  If there exists one group of data points (with more than 1 point) between some two jump points and  these data point have the same expected signal strength, then the two jumps are {of different directions} (up or down). Formally, let $j_k$ and $j_{k+1}$ be two jump points and ${\mu_{j_k}^*} = \ldots = \mu_{j_{k+1} -1}^*$, then ${(\mu_{j_k}^* - \mu_{j_k - 1}^*)(\mu_{j_{k + 1}}^* - \mu_{j_{k+1} - 1}^*) < 0}$.

    \end{enumerate}
\end{theorem}

\begin{proof}
From Theorem \ref{thm:SIC}, if condition (1) in Theorem \ref{thm:ic} holds, a stronger version of {the} irrepresentable condition holds {and thus the} irrepresentable condition  \eqref{IC} holds. If condition (1) does not hold, then there exists two jump points $j_k$ and $j_{k+1}$ such that $j_{k+1} \geq j_k + {2}$ and $\mu_{j_k}^* = \ldots = \mu_{j_{k+1} -1}^*$. From Equation \eqref{eqn:samesign}  in the proof of {Theorem \ref{thm:SIC}}, we see that {the} irrepresentable condition \eqref{IC} holds if and only if $\tilde \theta_k^* $ and $\tilde \theta_{k+1}^*$ {have} different signs. By the definition of $\tilde \theta$, we see that
${(\mu_{j_k}^* - \mu_{j_k - 1}^*)(\mu_{j_{k + 1}}^* - \mu_{j_{k+1} - 1}^*) < 0}$ is equivalent to $\tilde \theta_k^* $ and $\tilde \theta_{k+1}^*$ {having} different signs.
\end{proof}

Theorem \ref{thm:ic} says that only a few configurations of $\mu^*$ make {the} irrepresentable condition hold. In practice, a lot of signal patterns do not satisfy {either of} the two conditions {listed} in Theorem \ref{thm:ic}. For the Lasso problem, to comply with the irrepresentable condition, \cite{jia2012preconditioning} proposed a Puffer {T}ransformation. We now introduce {the} Puffer Transformation and apply it to solve the {fused Lasso} problem, {which we call the preconditioned fused Lasso.}

\section{{ Preconditioned Fused Lasso}}
\label{sec:PT}

\cite{jia2012preconditioning} introduces {the} Puffer {T}ransformation to the Lasso when the design matrix does not satisfy the {irrepresentable condition}. They showed that when $n\geq p$, even {if}  the Lasso is not sign consistent, after {the} Puffer Transformation, the Lasso is sign consistent under some mild conditions.

We assume that the design matrix $\X \in \R^{ n \times p}$ has rank $d = \min\{n,p\}$.  {By the} singular value decomposition, there exist matrices $U \in \R^{n \times d}$ and $V\in \R^{p \times d}$  with $U^TU=V^TV = I_d$ and a diagonal matrix $D \in \R^{d \times d}  $ such that $\X = UDV'$.  Define the \textbf{Puffer Transformation} \citep{jia2012preconditioning},
\begin{equation} \label{eqn:PT}
 F_{n\times n}  = UD^{-1}U^T.
\end{equation}
The preconditioned design matrix $ F \X$ has the same singular vectors as $\X$.  However, all of the nonzero singular values of $ F \X$ are set to unity: $ F \X = UV'$.  When $n\ge p$, the columns of $ F \X$ are orthonormal.  When $n \le p$, the rows of $ F \X$ are orthonormal.  \cite{jia2012preconditioning}
has a non-asymptotic result for the Lasso on $(F\X,FY)$ stated as follows.

\begin{theorem}[\cite{jia2012preconditioning}]
\label{thm:fixeddim}
Suppose that data $(\X,\Y)$ follow a linear model $Y = \X\truebeta + \epsilon$, where $Y = (y_1,\ldots,y_n)^T \in \R^{n\times1}$, $X \in \R^{n\times p}$ with ${x_i^T}$ as its $i$th row, $\beta^*\in \R^{p\times 1}$ and $\epsilon = (\epsilon_1,\ldots,\epsilon_n)^T \in \R^{n\times 1}$ with $\epsilon\sim N(0,\sigma^2I_n)$. Define the singular value decomposition of $\X$ as $\X = UDV'$.  Suppose that $n\geq p$ and $\X$ has rank $p$. We further assume that {the minimal eigenvalue} $\Lambda_{\min}(\frac{1}{n}X'X) \geq \tilde C_{\min}>0$. Define the {Puffer Transformation},
$ F  = UD^{-1}U^T.$ Let $Z= F \X$ and $a = F Y$. Define

\[\tilde \beta(\lambda) = \mathop{\text{argmin}} \limits_{b} \frac{1}{2} \|a - Z b\|_2^2 + \lambda  \| b\|_1.\]

If $\min_{j\in S}|{\beta_j^{*}}| \geq 2  \lambda$,
then with probability greater than
\begin{equation}
\label{eqn:lowerb}
1-2p \exp\left\{-\frac{n\lam^2\tilde C_{\min}}{2\sigma^2}\right\}
\end{equation}
 $\tilde \beta(\lambda) =_s \truebeta$.
\end{theorem}

The proof of Theorem  \ref{thm:fixeddim} can be found in \cite{jia2012preconditioning}. From the proof we see that the assumption that $\epsilon \sim N(0,\sigma^2I_n)$ can be {relaxed} to $\epsilon \sim N(0,\Sigma)$ with $\max_i{\Sigma_{ii}} \leq \sigma^2$.  Compare Theorem \ref{thm:fixeddim} to Theorem \ref{thm:consis}, we see that with {the Puffer Transformation}, the Lasso does not need the irrepresentable condtion any more.

The {FLSA} problem can be transformed to a standard Lasso problem. We have already shown that for most configurations of $\mu^*$, the design matrix  $\tilde X$ does not satisfy the irrepresentable condition. Now we turn to {the} Puffer {T}ransformation and obtain a concrete non-asymptotic result for {the} preconditioned fused Lasso. First  we have the following result on the singular values of $\tilde X$.

\begin{lemma}
\label{LEMMA:MIN.SING}
 $\tilde X \in \R^{n\times (n-1)}$ is defined in Equation \eqref{eqn:centered}. Let $\sigma_j(\cdot)$ denote the $j$-th largest singular value of a matrix.Then
 \[\sigma_1(\tilde X) \geq \sigma_2(\tilde X) \geq \cdots\geq  \sigma_{n-1}(\tilde X) \geq 0.5. \]

\end{lemma}

A proof of {Lemma} \ref{LEMMA:MIN.SING} can be found in the appendix. With the lower bound on singular values of $\tilde X$ and applying Theorem \ref{thm:fixeddim}, we have the following result {for our preconditioned fused Lasso}.
\begin{theorem}
\label{thm:FLSA}
Assume $y=(y_1,\ldots,y_n)$ satisfies model \eqref{eqn:model}. $\tilde X$ and $\tilde Y$ are defined in Equation \eqref{eqn:centered}. {Let} $\theta^* = A^{-1} \mu^*$, (equivalently, $\theta_1^* = \mu_1^*, \theta_i ^*= \mu_i ^*- \mu_{i-1}^*, i =2,\ldots,n$), where $A$ is defined in Equation \eqref{eqn:A}.  Let $\tilde \theta^* \in \R^{n-1}= (\theta_2^*,\theta_3^*,\ldots,\theta_n^*)^T$. Define the singular value decomposition of $\tilde X$ as $\tilde X = UDV'$.  {Denote} the {Puffer Transformation},
$ F  = UD^{-1}U^T.$ Let $Z= F \tilde X$ and $a = F \tilde Y$. Define

\begin{equation}
\label{eqn:FLasso}
\tilde \beta(\lambda) = {\mathop{\text{argmin}} \limits_{b}} \frac{1}{2} \| a - Z b\|_2^2 + \lambda  \| b\|_1.
\end{equation}

If $\min_{j\geq 2, \theta_j^*\neq 0}|\theta_j^*| \geq 2  \lambda$,
then with probability greater than
$$1-2n\exp\left\{-\frac{\lam^2}{8\sigma^2}\right\} $$
 $\tilde \beta(\lambda) =_s \tilde \theta^*$.
\end{theorem}

\begin{proof}
By Equation \eqref{eqn:LM}
\begin{equation*}
\tilde Y =  \tilde X \tilde \theta^* + \tilde \epsilon,
\end{equation*}
where $\tilde \epsilon = \epsilon - \bar{\epsilon}$ with  $E(\tilde \epsilon) = 0$.
$$var(\tilde\epsilon_i) = var(\epsilon_i - { \bar{\epsilon}}) = \frac{n-1}{n}\sigma^2 \leq \sigma^2.$$
According to the comments below Theorem \ref{thm:fixeddim}, we can apply  Theorem \ref{thm:fixeddim} to have a lower bound on $P(\tilde \beta(\lambda) =_s \tilde \theta^*)$. Let $s_1\leq s_2\leq\dots \leq s_n$ be the singular values of $\tilde X$.  From Lemma \ref{LEMMA:MIN.SING}, $s_1 \geq 0.5$. So
$\Lambda_{min}(\frac{1}{n} \tilde X' \tilde X) {=} \frac{s_1^2}{n} \geq \frac{1}{4n}$. Put $\tilde C_{\min} = \frac{1}{4n}$ in  expression \eqref{eqn:lowerb} and note that $\tilde X$ has $n-1$ columns, we have

$$P(\tilde \beta(\lambda) =_s \tilde \theta^*) \geq 1-2(n-1) \exp\left\{-\frac{\lam^2}{8\sigma^2}\right\}\geq 1-2n \exp\left\{-\frac{\lam^2}{8\sigma^2}\right\}.$$
\end{proof}

By the relationship between $\theta^*$ and $\mu^*$, if $\tilde \beta(\lambda)$ -- the estimate of $\tilde \theta^*$ has {the} sign recovery property, then the estimate of $\mu^*$ defined as follows has the property of pattern recovery.

 \begin{equation}
 \label{eqn:est}
 \hat\mu^*  = A \hat \theta^*
  \end{equation}
 with \[\hat \theta^* = [\hat\theta_1,\tilde \beta(\lambda)] \mbox{  and  }\hat\theta_1 = \bar Y - \bar X \tilde \beta(\lambda).
 \]
Theorem \ref{thm:FLSA} shows that the ability of pattern recovery depends on the signal jump strength ($\min_{j\geq 2, \theta_j^*\neq 0}|\theta_j^*|$) and the noise level $\sigma^2$. To get a pattern{-}consistent estimate, we need $\sigma$  small enough and $\min_{j\geq 2, \theta_j^*\neq 0}|\theta_j^*|$ big enough. To think about the small $\sigma^2$ issue, we can treat each $y_i$ as an average of multiple Gaussian measurements. If the number {of} measurements is $m$, then $\sigma^2 = \frac{\sigma_0^2}{m}$ with some constant $\sigma_0^2$.  If $m \gg \log(n)$, we can find a very small $\lambda$ to make the estimator defined in Equation \eqref{eqn:est} have {the} pattern recovery property. One choice of $\lambda$ is such that $\lambda^2 = \frac{\log(n+1)}{\sqrt{m}}$. For this choice of $\lambda$,  the probability of $\hat\mu^* =_{js} \mu^*$ is greater than $1-2\exp\left(-[\frac{\sqrt{m}}{{8}\sigma_0^2}-1]\log(n+1)\right)$, which goes to 1 as $m$ goes to $\infty$.

In the next section, we use simulations to illustrate that for general signal patterns, the {FLSA} does not have the pattern recovery property while the {preconditioned fused Lasso} has, which enhances our findings.

\section{Simulations}
We use simulation examples to confirm our theorems. We first set the model to be
\begin{equation}
\label{eqn:data}
y_i = \mu^*_i + \epsilon_i,
\end{equation}
where
$$\mu^*_i =   \left\{ \begin{array}{cc}
      0, & 1\leq i \leq 100 \\
      -2,& 101\leq i \leq 110 \\
      -0.1, & 111\leq i \leq 210 \\
      2,& 211\leq i \leq 220 \\
      0.1, & 221\leq i \leq 320 \\
      -2,& 321\leq i \leq 330 \\
      0, & 331\leq i \leq 430 \\
 \end{array}\right.$$
and the errors are i.i.d.\ Gaussian variables with mean $0$ and standard deviation $\sigma = 0.25$. This one is similar to the example in \cite{rinaldo2009properties} except that the noise here is larger ($\sigma = 0.2$ in  \cite{rinaldo2009properties}). Figure \ref{fig:data} shows one sequence of {sample} data (points) along with the true expected signal (lines).

\begin{figure}[h!]
\begin{center}
\includegraphics[scale = 0.6]{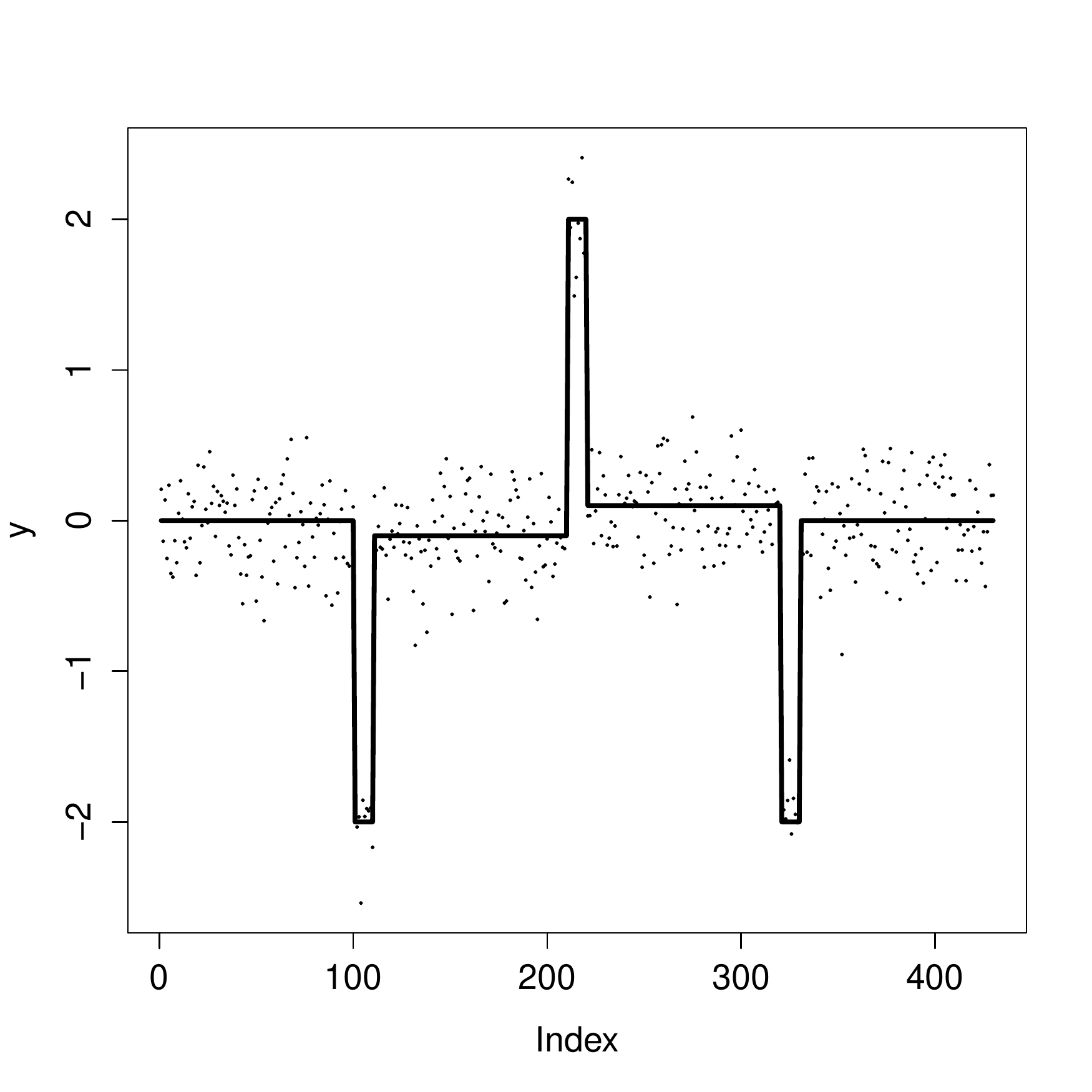}
\end{center}
\caption{\label{fig:data}{Sample} data (points) and the expected signals (lines).}
\end{figure}

From Figure \ref{fig:data} we see that the data points are grouped into seven clusters and featured by three spikes. The points can be well separated due to small noise.
{
We will use this typical example to compare the performances of the two methods, the FLSA and the preconditioned fused Lasso, in recovering the signal patterns. There are many criteria that can be used in comparison. In the context of pattern recovery, it is natural to {define a loss function, which we call the \textit{pattern loss} $(\ell_{PA})$ of the recovered sequence of signals $\hat{\mu}^* = (\hat{\mu}^*_1, \hat{\mu}^*_2, \ldots, \hat{\mu}^*_n) \in \mathbb{R}^n$ as follows:
\begin{equation*}
  {\ell_{PA}}(\hat{\mu}^*) = |\{i: \text{sign}(\hat{\mu}^*_{i+1} - \hat{\mu}^*_i) \neq \text{sign}(\mu^*_{i+1} - \mu^*_{i}), i = 1, \ldots, n-1\}|
\end{equation*}
where ${\mu}^* = ({\mu}^*_1, {\mu}^*_2, \ldots, {\mu}^*_n)$ is the expected signals and $|\cdot|$ the cardinality of a set. Note that the pattern loss achieves 0 if and only if the pattern of the signals is recovered exactly. We compare the solutions under the two methods (FLSA and preconditioned fused lasso). For each method, the solution chosen is the one that minimizes the {pattern loss on the solution path}.}

We first apply {the FLSA} to estimate $\mu_i^*,i=1,2,\ldots,  {430}$.  When calculating the  {FLSA} {solution}, we use a path algorithm proposed by \cite{hoefling2010path} which is  {very efficient} to give the whole solution path {of the FLSA}. An R package (``flsa") for this algorithm { is available} in http://cran.r-project.org/web/packages/flsa/index.html. In fact, the whole  {FLSA} solution path is piecewise linear in $\lambda$. ``flsa" only stores the sequence of $\lambda$'s when the direction {of} the linear function changes. {Note that the {pattern loss} does not change with $\lambda$ on every linear piece of the solution path.} { By} {comparing} the signal pattern {of all the} estimated signals on the solution path with the true {signals} $\mu^*$, we see that there is no one solution {that} recovers the original signal pattern. {That is, all the {FLSA solutions} have a positive {pattern loss}. We present in Figure  \ref{fig:fslasolutions} (left panel) the solution that minimizes {such loss}. We see that this estimate is just the trivial estimate that averages all the signals, which obviously does not give satisfactory recovery of signal patterns.}

\begin{figure}
\centering
\mbox{\subfigure{\includegraphics[scale = 0.4]{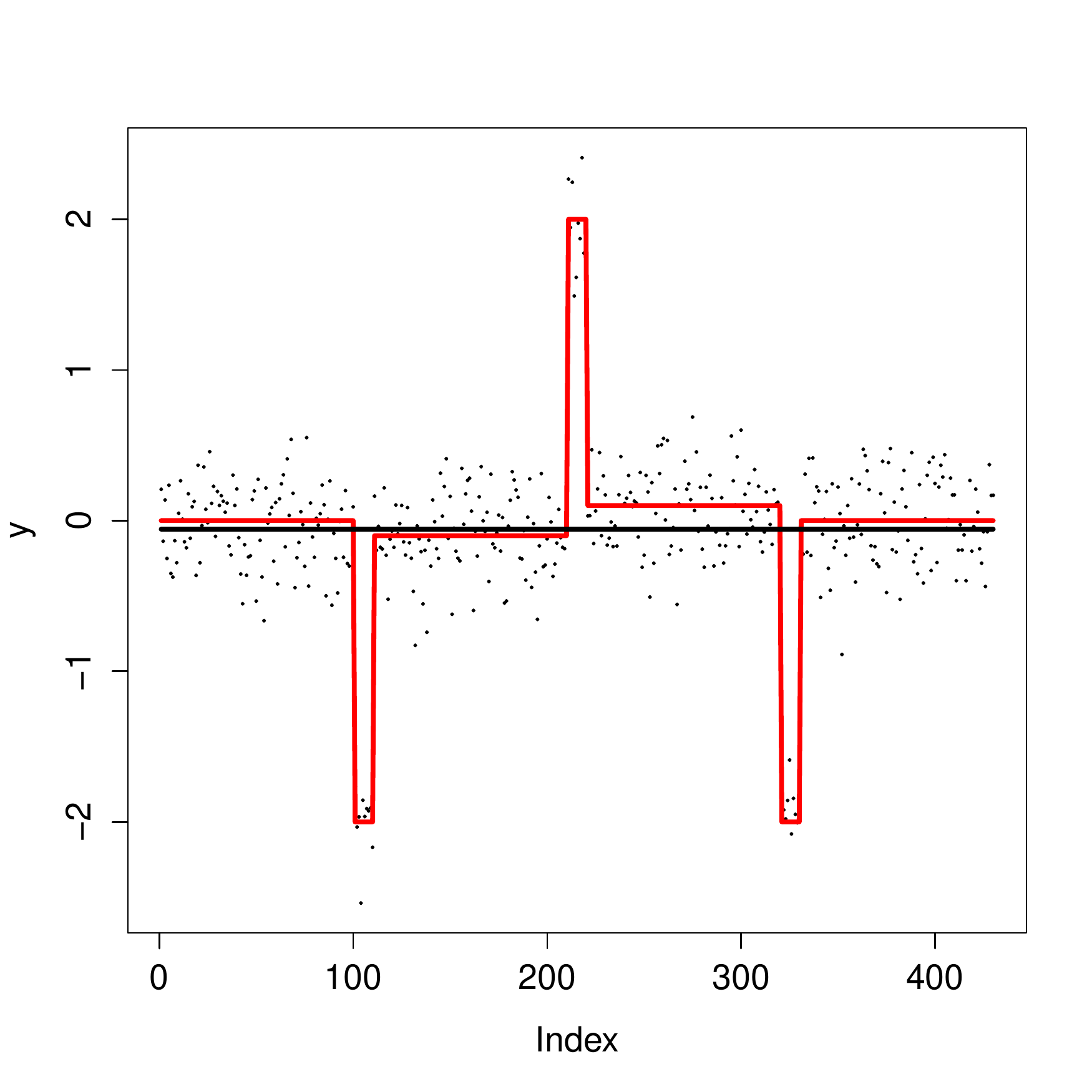}}}\quad
\subfigure{\includegraphics[scale = 0.4]{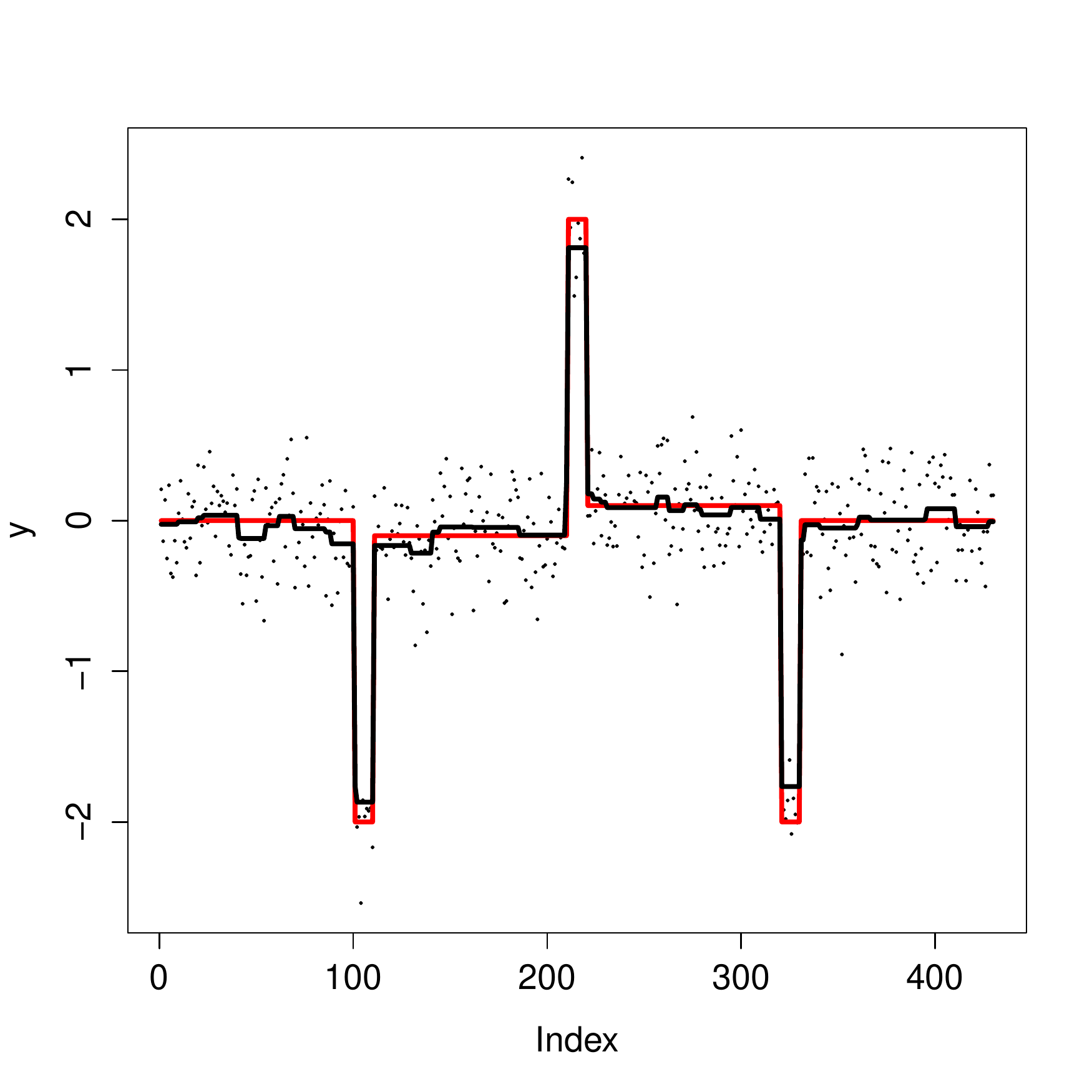}}
\caption{FLSA solutions. Left panel: the {FLSA solution}(in black lines) with tuning parameter selected by {minimizing the pattern loss} of $\hat \mu^*$; Right panel: the {FLSA solution}(in black lines) with tuning parameter selected by minimizing the $\ell_2$ error between $\hat \mu^*$ and $\mu^*$. { {The} red lines are the expected signal sequence. }\label{fig:fslasolutions}}
\end{figure}


For each $\lambda$ in the sequence, we also calculate the common $\ell_2$ distance between {the} estimated {signals} $\mu^*$ and the true {ones}. The estimate with the smallest $\ell_2$ distance is reported in Figure \ref{fig:fslasolutions} (right panel). We see that for this estimate, it does not recover the original signal pattern either.


{To compare}, we calculate the solution of the Lasso defined in Equation \eqref{eqn:FLasso}. After the SVD and the Puffer Transformation, this becomes much easier. We only need to do a soft-thresholding with the given $\lambda$. This is because
$$
Z^T Z = X^T F^T F X = (VDU^T)(UD^{-1}U^T)(UD^{-1}U^T)(UDV^T) = I_n
$$
and the property of the Lasso allows us to solve it directly by soft-thresholding
$$\hat{b}(\lambda) = SH_\lambda(Z^T a). $$
Obviously, $\hat{b}(\lambda)$ is also a piecewise linear function in $\lambda$ and the break points are $\lambda_i= |Z^Ta|_{(i)}, i =1,2,\ldots, n$, where $x_{(i)}$ denotes the $i$th largest value in vector $x$ and $n$ is the dimension of vector $Z^Ta$. On the solution path, for each $\lambda_i$ we have an estimate $\hat{\mu}$ for $\mu^*$. {By examining the {pattern loss} of the solutions on the path, we find that there is one solution { (in fact any solution on {that linear part between the one chosen and the one at the previous breakpoint)} has 0 loss, which means it recovers the signal pattern exactly. We report this solution in Figure \ref{fig:PT-LASSO}.}

\begin{figure}[h!]
\begin{center}
\includegraphics[scale = 0.6]{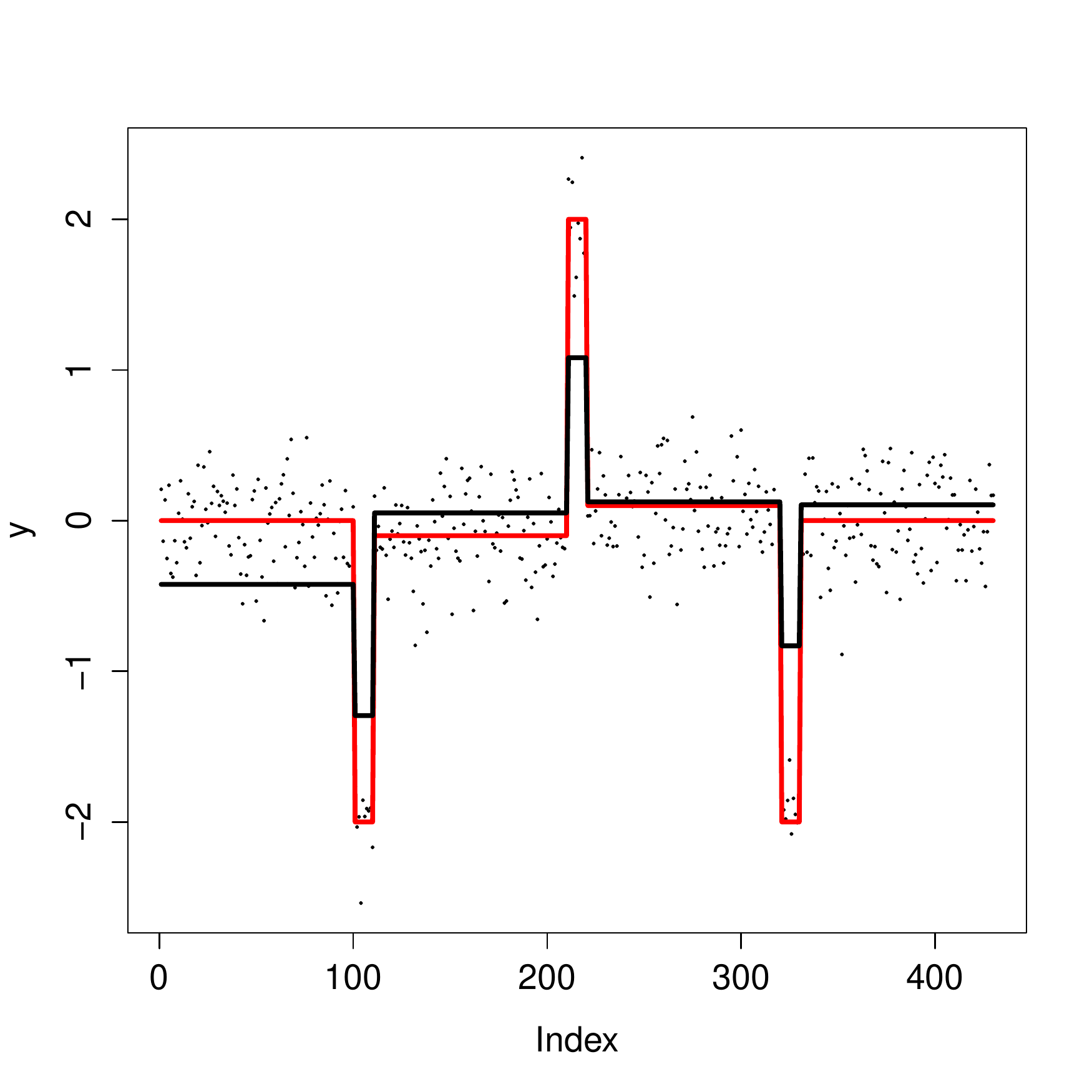}
\end{center}
\caption{\label{fig:PT-LASSO}
The preconditioned fused Lasso estimator (in black lines).  Tuning parameter is selected by  finding a solution which exactly recovers the signal pattern. {{The} red {line is} the expected signal sequence. }}
\end{figure}

Note that the reported estimate is very biased from the expected value. There is a tradeoff between the {unbiasedness} and {the quality of pattern recovery}. One possible solution for {the unbiasedness} is via {a} two{-}stage estimator-- for the first stage the signal patten is recovered and for the second stage an unbiased estimate is obtained.  

The above example just gives one data set to compare the performances of the  {FLSA} and preconditioned fused Lasso. We now randomly draw {1000 datasets} and compare the approximate probability  (denoted as $P$) that there exists a $\lambda$ such that the pattern of the signals can be completely recovered. The results are as follows.
\begin{itemize}
  \item FLSA: $P \approx 0$.
  \item Preconditioned Fused Lasso $P \approx  0.926$.
\end{itemize}

This example again illustrates the strength of our algorithm in pattern recovery of {blocky} signals. Nevertheless, as intuitively, it {loses power} like other recovering algorithms when the noise level becomes stronger  {and makes} it difficult to tell the boundaries between {the} blocks. Our theorem also reflects this relationship between recovery probability and noise level. In the next example, we change $\sigma$ from $0.1$ to ${0.4}$ and compute the probability of pattern recovery. {For each $\sigma$, we randomly draw {$1000$} datasets following model described in Equation \eqref{eqn:data} and obtain the estimated probability via the proportion that    there exists a $\lambda$ such that the pattern of the signals can be completely recovered. }The estimated probabilities are reported in Figure  \ref{fig:flsa sigma}, from which {we see that the probability of pattern recovery under small noise is extremely high but this cannot hold when the signals are corrupted by stronger noise, which makes the boundaries between groups vague and hard to distinguish.}

\begin{figure}[h!]
\begin{center}
\includegraphics[scale = 0.6]{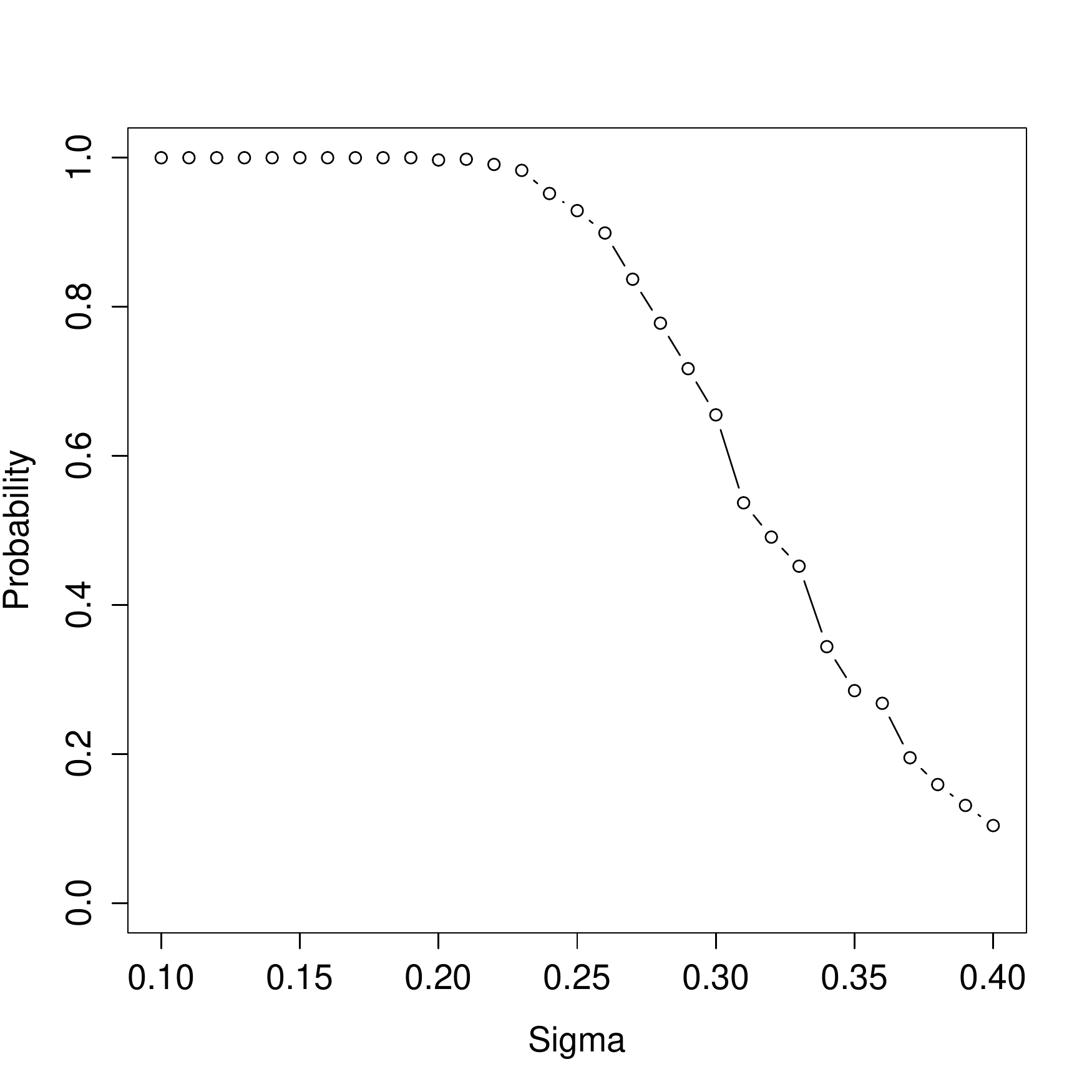}
\end{center}
\caption{\label{fig:flsa sigma}
The {estimated probability} of pattern recovery under different noise level for the  {preconditioned fused Lasso}. { Each point is estimated with {$1000$} randomly generated datasets.}}
\end{figure}

\section{Conclusions {and Discussions}}
In this paper we provided more understanding {of {the} FLSA and shed some  light on the {insight of the FLSA}}. {The} FLSA can be transformed to a standard Lasso problem. The sign recovery of the transformed Lasso problem is equivalent to  {the} pattern recovery of the  {FLSA} problem. Theoretical analysis showed that the transformed Lasso problem is not sign consisten{t} in most  {situations}. So  {the} FLSA might {also meet this consistency problem} when it is used to recover signal patterns. {To overcome such problem}, we introduced {the preconditioned fused} Lasso. We gave non-asymptotic results on the preconditioned fused Lasso. The {result} implies that when the noise is weak, the preconditioned fused Lasso can recover the signal pattern with high probability. We also found that the preconditioned fused Lasso is sensitive to the scale of the noise level. Simulation studies {confirmed} our findings.

One may argue that we only considered the pattern recovery using the fusion regularization parameter $\lambda_2$ and that the sparsity regularization parameter $\lambda_1$  {can} be used to adjust to the right pattern.  However, remember that the main purpose of introducing $\lambda_1$ is for the sparsity of the model. It is not statistically reasonable to use this regularization parameter only to recover the blocks.

If considered in the context of both sparsity and block recovery, this is impossible in most time. Using the example above, we claim that the pattern and sparsity cannot be recovered at the same time by {FLSA}. We know from \cite{friedman2007pathwise} that as long as two parameters are fused together for some $\lambda_2^{(0)}$, it will be fused for all $\lambda_2> \lambda_2^{(0)}$. This implies inversely that if two $\hat\mu_i's$ are partitioned for some big $\lambda_2$, they will not be fused for all $\tilde{\lambda_2} < \lambda_2$. {Let} us focus on the first partition as $\lambda_2$ decreases {from some large value {when} all the {estimated} parameters are fused together}. We found it happens at Point 210, which was not a jump point in the original signal sequence $\mu^*$. Then Point 209 ($\hat \mu_{209}\approx -0.063$) and Point 210($\hat \mu_{210}\approx -0.052$) can never be fused together when $\lambda_2$ goes down. The only way to make them together is to do drag both of them to zero in the soft-thresholding step. But they are nonzero signals($\mu_{209}^* = \mu_{210}^*=-0.1$) and the sparsity recovery will be clearly violated if doing so.

We claim that a good pattern recovery will facilitate things afterwards.  The  {preconditioned fused Lasso} is reliable for pattern recovery,  and so it can be incorporated into other processes -- such as  the recovery of sparsity.


\appendix
\begin{center}
{\large Appendix}
\end{center}
We prove some of our theorems in the appendix.
\section{Proof of Theorem \ref{THM:NECESSARY CONDITION}}
We first give a well-known result that makes sure the Lasso exactly recovers the sparse pattern of $\truebeta$, that is
$\hbet =_s \truebeta$. The following Lemma gives necessary
and sufficient conditions for $\mbox{sign}(\hbet) =
\mbox{sign}(\truebeta)$, {which follows from the KKT conditions}. The proof of this lemma can be found in  {\cite{wainwright2009}}.

\begin{lemma}
\label{KKT}
For {the} linear model $Y=\X\truebeta+\epsilon$,
assume that the matrix $\XaT\Xa$ is invertible. Then for any given
$\lam>0$ and any noise term $\e\in \R^n$, there exists a Lasso
estimate $\hbet$ described in Equation (\ref{eqn:Lasso}) which satisfies $\hbet=_s\truebeta$, if and only if
the following two conditions hold
\begin{equation}
\left|\XbT\Xa(\XaT\Xa)^{-1}\left[\XaT\e-\lam
 sign(\betA)\right]-\XbT\e\right|\leq \lam, \label{R1}
\end{equation}
\begin{equation}
sign\left(\betA+(\XaT\Xa)^{-1}\left[\XaT\e-\lam
sign(\betA)\right]\right)=sign(\betA), \label{R2}
\end{equation}
where the vector inequality and equality are taken elementwise.
Moreover, if the inequality (\ref{R1}) holds strictly, then
$$\hat{\beta}=(\hbetA,0)$$ is the unique optimal solution to the Lasso problem in Equation (\ref{eqn:Lasso}), where
\begin{equation}
\hbetA=\betA+(\XaT\Xa)^{-1}\left[\XaT\e-\lam
sign(\truebeta)\right]. \label{R3}
\end{equation}
\end{lemma}

{\textbf{Remarks.} } As in \cite{wainwright2009}, we state an equivalent  condition for
(\ref{R1}). Define
$$\overrightarrow{b}=sign(\betA),$$ and define
$$V_j=X_j^T\left\{\Xa(\XaT\Xa)^{-1}\lam\overrightarrow{b}-\left[\Xa(\XaT\Xa)^{-1}\XaT-I)\right] {\e}\right\}.$$
By rearranging terms, it is easy to see that (\ref{R1}) holds
 if and only if
\begin{equation}
\mathcal{M}(V)=\left\{\max_{j\in S^c} |V_j|\leq \lam\right\} \label{C1}
\end{equation} holds.

{With Lemma \ref{KKT} and the above comments, now we prove Theorem \ref{THM:NECESSARY CONDITION}. Without loss of generality, assume for some $j \in S^c$ and $\zeta\geq 0$,
$$X_{j}^T\Xa\left(\Xa^T\Xa\right)^{-1}\overrightarrow{b}=1+\zeta.$$
Then
$$V_{j}=\lam(1+\zeta)+\tilde{V}_j,$$
where $\tilde{V}_j=-{X_j^T}[\Xa\left(\Xa^T\Xa\right)^{-1}\Xa^T-I]\frac{\e}{n}$ is a Gaussian random variable with mean $0$, so $P(\tilde{V}_j>0)=\frac{1}{2}$. Therefore,
$$P(V_j>\lam)\geq \frac{1}{2}$$
and the equality holds when $\zeta = 0$. This implies that for any $\lam$, Condition \eqref{R1} (a necessary
condition) is violated with probability greater than $1/2$.
}

In the proof of Theorem \ref{thm:SIC}, we need an algebra result as follows.
\begin{lemma}
\label{lemmaapp}
    For $k \geq 3, a_1, \ldots, a_k \in \mathbb{R}$ and are not equal to each other. $A = (a_{ij})_{k \times k},$ with $a_{ij} = a_\ell$ where $\ell = \text{max}\{i, j\}$. That is,
  \begin{equation*}
    A = \left(\begin{array}{cccc}
        a_1 & a_2 & \vdots & a_k \\
        a_2 & a_2 & \vdots & \vdots \\
        \cdots & \cdots & \cdot & \vdots \\
        a_k & \cdots & \cdots & a_k
        \end{array}\right)
  \end{equation*}
{  Then the inverse of $A$
 \begin{equation*}
    (A)^{-1} = \left[\begin{array}{cccccc}
       r_{11} & r_{12} &  &  &  & \\
         r_{21} & r_{22} &  r_{23} & & &\\
         & r_{32} & r_{33}  &  r_{34}  & &\\
         & & \ddots & \ddots & \ddots & \\
         & & & r_{k-1,k-2} & r_{k-1,k-1} & r_{k-1,k} \\
         & & & & r_{k,k-1} & r_{k,k}
         \end{array}\right]
\end{equation*}
where
\begin{equation*}
      r_{ij} = \begin{cases}
      \frac{1}{a_1 - a_2} & i = j = 1 \\
      -\frac{1}{a_{j-1} - a_{j}} & i = j - 1 \\
      -\frac{1}{a_j - a_{j+1}} & i = j + 1 \\
      \frac{a_{j-1} - a_{j+1}}{(a_{j-1} - a_j)(a_j - a_{j+1})} & 1 < i = j < k \\
      \frac{a_{k-1}}{(a_{k-1} - a_{k})(a_k)} & i = j = k \\
      0 & \text{otherwise}.
      \end{cases}
\end{equation*}
}
\end{lemma}

\begin{proof}
This lemma can be directly verified via the following equations:
\[\sum_i a_{ji}r_{ij} = 1 \ \ \mbox{and} \ \ \sum_{i} a_{\ell i}r_{ij} = 0, \ell \neq j.\]
We first verify  $\sum_i a_{ji}r_{ij} = 1 $, for all $j$.\\
When $j = 1$,
$$\sum_{i} a_{1i}r_{i1} = a_1 r_{11} + a_2 r_{21} = a1 \cdot \frac{1}{a_1 - a_2} + a_2 \cdot \frac{-1}{a_1 - a_2} = 1.$$
When $1<j<k$,
\begin{eqnarray*}
\sum_{i} a_{ji}r_{ij} &=& a_{j,j-1} r_{j-1,j} + a_{j,j} r_{j,j} + a_{j,j+1} r_{j+1,j} \\
& = & a_j \cdot \frac{-1}{a_{j-1} - a_{j}} + a_j \cdot \frac{a_{j-1} - a_{j+1}}{(a_{j-1} - a_j)(a_j - a_{j+1})} + a_{j+1} \cdot\frac{-1}{a_{j} - a_{j+1}} \\
& = & 1.
\end{eqnarray*}
When $j = k$,
\begin{eqnarray*}
  \sum_{i} a_{ji}r_{ij} & = & a_{k,k-1}  r_{k-1,k} + a_{k,k} r_{k,k} \\
    & = & a_k \cdot \frac{-1}{a_{k-1} - a_{k}} + a_k \cdot \frac{a_{k-1}}{(a_{k-1} -a _{k})a_k} \\
    & = & 1.
\end{eqnarray*}

We next verify $\sum_{i} a_{\ell i}r_{ij} = 0 \ \ \mbox{for all} \ \ \ell \neq j$. We only very the general case when there are three elements in one column of $A^{-1}$. The other verifications are the same.
$\sum_{i} a_{\ell i}r_{ij}  = a_{\ell, j-1} r_{j-1,j} + a_{\ell, j}  r_{j,j} + a_{\ell,j+1} r_{j+1,j}$.
Since $\ell \neq j$, there are only two situations we need to consider. (1) $\ell \leq j-1$ and (2) $\ell \geq j+1$. \\
When $\ell \leq j-1$,
\begin{eqnarray*}
  \sum_{i} a_{\ell i}r_{ij}  &=& a_{\ell, j-1} r_{j-1,j} + a_{\ell, j}  r_{j,j} + a_{\ell,j+1} r_{j+1,j}\\
   &=& a_{ j-1} r_{j-1,j} + a_{j}  r_{j,j} + a_{j+1} r_{j+1,j} \\
   &=& 0.
\end{eqnarray*}
When $ \ell \geq j+1$,
\begin{eqnarray*}
  \sum_{i} a_{\ell i}r_{ij} & = & a_{\ell, j-1} r_{j-1,j} + a_{\ell, j}  r_{j,j} + a_{\ell,j+1} r_{j+1,j} \\
  &=& a_{ \ell} r_{j-1,j} + a_{\ell}  r_{j,j} + a_{\ell} r_{j+1,j} \\
  &=& a_{ \ell}\cdot [\frac{-1}{a_{j-1} - a_j} +  \frac{a_{j-1} - a_{j+1}}{(a_{j-1} - a_j)(a_j - a_{j+1})} +  \frac{-1}{a_j - a_{j+1}} ]\\
  &=& 0.
\end{eqnarray*}

\end{proof}

%
%
%
\section{{Proof of Lemma \ref{LEMMA:MIN.SING}}}

To prove Lemma \ref{LEMMA:MIN.SING}, we need the following two results.

\begin{lemma}
\label{lemma:1/2}
Let  $X\in \R^{n\times n}$ be a lower triangular matrix with elements 1 on and below the diagonals and 0 in other places.
\begin{equation*}
\label{eqn:X1}
  X_{ij} = \begin{cases}
    1 & i \geq j \\
    0 & i < j.
  \end{cases}
\end{equation*}
The minimal singular value is {greater or equal to} 0.5.
\end{lemma}

\begin{proof}
Let $X =  (a_{ij}) \in \mathbb{R}^{n\times n}$ be the matrix satisfying the condition of the lemma. {Note that} the singular values of this matrix $X$ are the non-negative square roots of the eigenvalues of $X^T X$. {Hence} {it} suffices to show that all the eigenvalues of $X^T X$ are greater or equal to 0.25.

The explicit expression of $C_{n} = X^T X = (c_{ij}) \in \mathbb{R}^{n \times n}$ is
\begin{equation*}
  c_{ij} = n + 1 - \max\{i, j\}.
\end{equation*}

By Lemma \ref{lemmaapp}, we have

 \begin{equation*}
    C_n^{-1} = \left[\begin{array}{cccccc}
       1 & -1 &  &  &  & \\
         -1 & 2 & -1 & & &\\
         & -1 & 2  &  -1  & &\\
         & & \ddots & \ddots & \ddots & \\
         & & &-1&2 & -1 \\
         & & & & -1& 2
         \end{array}\right].
\end{equation*}

Then for any vector $u\in \R^{n\times 1}$,
\begin{eqnarray*}
u^T C_n^{-1} u &=& u_1^2 + \sum_{i=2}^n (2u_i^2) - 2\sum_{i=1}^{n-1} u_i u_{i+1}\\
&\leq& 2\sum_{i=1}^n u_i^2 - 2\sum_{i=1}^{n-1} u_i u_{i+1}\\
&\leq& 2\sum_{i=1}^n u_i^2+ 2\sum_{i=1}^{n-1} |u_i u_{i+1}|.
\end{eqnarray*}

By the fact that $\sum_{i=1}^{n-1} |u_iu_{i+1}| \leq \frac{1}{2}\sum_{i=1}^{n-1}(u_i^2 + u_{i+1}^2) \leq \sum_{i=1}^n u_i^2$, we have
\[u^T C_n^{-1} u  \leq 4 \sum_{i=1}^n u_i^2,\]
which implies that the eigenvalues of $C_n^{-1}$ {are less or equal to $4$ and thus the eigenvalues of $C_n$ are all greater or equal to 0.25.}

\end{proof}

The following lemma states the relationship between eigenvalues of second moments for centered and non-centered data.
Let $X\in \R^{n\times p}$ be a data matrix. Define the (empirical) covariance matrix of $X$ be
\[S = \frac{X_c'X_c}{n},\]
where $X_c$ is the centered version of $X$ with the $j$-th column of $X_c$ be $X_j - \bar{X}_j$.
Let the second moments of the data set $X$ be
\[T = \frac{X'X}{n}.\]
Then the eigenvalues of $S$ and $T$ have the following property.

\begin{lemma}[\cite{cadima2009relationships}]
\label{c-u}
Let $S$ be the covariance matrix for a given data set, and $T$ its corresponding matrix of non-central second moments. Let $\lambda_j(\cdot)$ be the $j$-th largest eigenvalue of a matrix. Then
\[\lambda_{j+1}(T) \leq \lambda_j(S) \leq \lambda_j(T).\]
\end{lemma}
Lemma \ref{c-u} can be found on page 5 in \cite{cadima2009relationships}.

With the results from Lemma \ref{lemma:1/2} and Lemma \ref{c-u}, we now prove Lemma \ref{LEMMA:MIN.SING}.
\begin{proof}
Let  $X\in \R^{n\times n}$ be a lower triangular matrix with elements 1 on and below the diagonals and 0 in other places.
\begin{equation*}
  X_{ij} = \begin{cases}
    1 & i \geq j \\
    0 & i < j.
  \end{cases}
\end{equation*}
Let $\sigma_j(\cdot)$ denote the $j$-th largest singular value of a matrix.
By Lemma \ref{lemma:1/2}, the smallest singular value is not less than 0.5, that is $\sigma_j(X) \geq 0.5, \forall j=1,\ldots,n$.  Now let $X_c$ be the centered version of $X$, then $X_c = [\mathbf{0},\tilde X]$, where
$\mathbf{0}$ is a column vector with all elements 0, and $\tilde X \in \R^{n\times (n-1)}$ as defined in Equation \eqref{eqn:centered}. Let $\sigma_j(\cdot)$ denote the $j$-th largest singular value of a matrix. By Lemma \ref{c-u}, we have
\[\sigma_{j+1}(X)\leq \sigma_j(X_c) \leq \sigma_j(X),\forall j=1,2,\ldots,n-1.\]
{In particular}, take $j=n-1$ in the above inequalities and we have $\sigma_{n-1}(X_c) \geq \sigma_{n}(X) \geq 0.5.$ Since $X_c$ is singular, the minimal singular value $\sigma_n(X_c) =0$. Therefore,
$$
    \sigma_{n-1}(\tilde X) = \sigma_{n-1}(X_c) \geq 0.5.
$$
\end{proof}

\newpage
\bibliographystyle{plainnat}
\bibliography{references}

\end{document}